\documentclass[A4paper,12pt]{article}

\usepackage{enumitem}
\usepackage{lipsum}

\usepackage{bm}
\usepackage{graphicx}
\usepackage{amssymb}  
\usepackage{amsthm}
\usepackage{amsmath}

\usepackage[utf8]{inputenc} 
\usepackage[T1]{fontenc}    
\usepackage{hyperref}       
\usepackage{url}            
\usepackage{booktabs}       
\usepackage{amsfonts}       
\usepackage{nicefrac}       
\usepackage{microtype}      

\newtheorem{theorem}{\sc Theorem}[section]

\newtheorem{proposition}[theorem]{\sc Proposition}

\title{A New Family of Near-metrics for Universal Similarity}

\author{
{\sc Chu Wang}
\and
  {\sc Iraj Saniee}
\and
{\sc William S. Kennedy}
\and
{\sc Chris A. White}
\thanks{Nokia Bell Labs, 600 Mountain Avenue, Murray Hill, NJ 07974 ({\tt chu.wang, iraj.saniee, william.kennedy, chris.white}@{\tt nokia-bell-labs.com}).
}
}

\begin{document}

\maketitle

\begin{abstract}
We propose a family of near-metrics based on local graph diffusion to 
capture similarity for a wide class of data sets.  These quasi-metametrics, 
as their
name suggests, dispense with one or two standard axioms of metric spaces, 
specifically distinguishability and symmetry, so that similarity between 
data points of arbitrary type and form could be measured broadly and 
effectively. The proposed near-metric family includes the forward $k$-step 
diffusion and its reverse, typically on the graph consisting 
of data objects and their features.
By construction, this family of near-metrics is particularly appropriate for 
categorical data, 
continuous data, and vector representations of images and text extracted 
via deep learning approaches.
We conduct extensive experiments to evaluate the performance of this
family of similarity measures and compare 
and contrast with traditional measures of similarity used for each specific application and with the ground truth when available.
We show that for structured data including categorical and continuous 
data, the near-metrics corresponding to normalized forward $k$-step  
diffusion ($k$ small) work as one of the best performing similarity measures; 
for vector representations of text and images including
those extracted from deep learning, the near-metrics derived from
normalized and reverse $k$-step graph diffusion ($k$ very small) exhibit 
outstanding ability to distinguish data points from different classes.
\end{abstract}

\section{Introduction}
A core requirement and a fundamental module in various machine learning tasks is to measure the similarity between data points.
Clustering, community detection, search and query, and recommendation algorithms, to 
name only a few, involve similarity derivations that are essentially constructed 
on top of underlying distance measures.
To this end, numerous similarities have been designed specifically for 
particular types of data or data sets in order 
to optimize the performance~\cite{boriah2008similarity,goshtasby2012similarity,ting2016overcoming}.

For categorical data sets, the most straightforward similarity measure is the 
overlap measure, which counts the number of matching attributes between two data 
points.
More delicate measures take into consideration the number of 
categories, feature frequencies, and other local or global information.
Frequently used measures for categorical data include but are not limited to Eskin, 
Lin, Goodall, and Occurrence Frequency.
We refer interested readers to \cite{boriah2008similarity} for a comprehensive review.

For continuous data sets, completely different similarity measures have been used.
Cosine similarity and the inner product similarity are textbook approaches~\cite{jones1987pictures}.
The Minkowski distance is widely used which includes both the Manhattan distance (order 1 version) and Euclidean distance 
(order 2 version).
If data points are generated from the same distribution, Mahalanobis 
distance, or more broadly Bregman divergence, is frequently 
adopted~\cite{de2000mahalanobis,bregman1967relaxation}.  For distance between data
sets or distributions, Bhattacharyya and Hellinger distances and Kullback-Leibler divergence
are often leveraged~\cite{bhattacharyya1946measure, le2012asymptotics,kullback1951information}.
Hamming distance continues to be used for bit-by-bit string comparison \cite{hamming1950error}.
At the risk of information loss, categorical similarities also apply if the data is discretized first.
The MDL method is the most well-known algorithm for such a discretization task~\cite{fayyad1993multi}.

Besides structured data types mentioned above, unstructured data, such as 
text, audio, image, and video, are even more commonly encountered~\cite{gandomi2015beyond}.
It is not straightforward to regard unstructured data as in a simple vector form, 
because of its sequential or geometrical properties~\cite{rieck2007computation,milioris2014joint,rieck2011similarity,rieck2008linear,goshtasby2012similarity}.
One needs to carefully transform such unstructured data into a vector 
representation before applying similarity measures for reasonable success.
For text data, the traditional tf-idf method captures the frequency information while loses the order information~\cite{salton1986introduction}.
More recent deep learning approaches like the ``word2vec'' model maps any given word into a dense, 
low-dimensional (usually no more than a few hundred tuple) vector~\cite{mikolov2013distributed}.
Deep-learning based vector representation is already making its way in various applications~\cite{le2014distributed,collobert2011natural,kim2014convolutional,turian2010word},
yet to the best of our knowledge, there is currently no related work on similarity measures designed for 
such vectors.

In this paper, we propose a family of near-metrics based on 
local graph diffusion to quantify a distance or likeness between data points.
Instead of focusing on a specific data set or data type, 
our motivation is to develop a family of distance-like functions, or near-metrics, 
that are sufficiently general to capture all the above data types in a 
single representation.  For $n$ data points, each of them with $m$ features, 
we regard the data set as a bipartite graph consisting of $n$ \textit{object} 
nodes and $m$ \textit{feature} nodes connected by $mn$ weighted edges.  
The edge weights measure 
the presence, extent or strength, of each feature in each data object.
The similarity of one data point to another can then be defined by the 
transition probability of the random walk on this bipartite graph from
one point to the other for a given number of steps.
Such a graph diffusion-based similarity has several natural variants 
according to the locality parameter $k$ which is the number of steps of the 
random walk (its order), the diffusion direction (the forward or reversed 
variant), 
and the normalization of the feature weights (the normalized variant) where 
the edge weights are rescaled at each incident node to add to 1.

\subsection{Key Contributions}
The near-metrics proposed are based on the abstract concept of diffusion
on the object-feature graph. Therefore, they are not limited to specific data types
or data sets and we expect them to work in a wide variety of areas.
We follow a dual path: first we analyze the graph diffusion similarity 
itself theoretically in order to uncover its properties and features;
second, we evaluate the performance of this family of distances on 
various data sets and data types in comparison with frequently used similarities and when possible, to the ground truth.
We summarize our key contributions as follows:

\begin{itemize}[leftmargin=*]

\item 
We propose a family of near-metrics. 
The intuition behind the localized (i.e., finite-step) graph 
diffusion similarity is the mass transfer analogy of a random walk
from one object to another via their common features: the more common
the features two objects have the more mass is transferred from an
object to the other.

\item
The family of near-metrics is analyzed theoretically in several respects.
We show that, when applied to categorical datasets or distributions, the 
graph diffusion is a metametric.
For general data sets, conditions for the proposed near-metrics to be metametrics or quasi-metametrics are thoroughly discussed.
The analysis demonstrates that the proposed near-metrics are able to 
function as well defined similarities, while giving enough flexibilities to accommodate different types of data. 

\item
We evaluate this family of near-metrics on a wide variety of benchmark 
datasets of all types including categorical, continuous, text, images, 
and their mixtures. 
For structured data sets, the proposed near-metrics are 
among the best performing measures of similarity.
For compact representations of unstructured data like texts and images,
the proposed near-metrics capture the intrinsic relation 
between data points and their performance is outstanding.

\end{itemize}

\subsection{Organization of the Paper}
The rest of the paper is organized as follows.
In Section \ref{sec:formula}, we formally introduce the graph diffusion similarity 
and its variants, discuss their basic properties, 
and provide necessary and sufficient conditions for these to be 
quasi (loss of symmetry) or metametrics (loss of complete distinguishability).
In Section \ref{sec:structured}, we evaluate the performance of the graph diffusion 
similarity together with frequently used similarity measures in various bench mark 
structured data sets.
In Section \ref{sec:unstructured}, the evaluation is further conducted for vector 
representations of unstructured data including texts and images.
Section \ref{sec:conclusion} discusses our results and observations, proposes future 
research directions, and concludes the paper.

\section{Graph Diffusion Similarity}\label{sec:formula}

Consider $n$ objects, each endowed with a non-negative feature vector of dimension $m$.
All the information about objects and their features is captured by a 
bipartite graph $\mathcal{B}$ of $n$ object nodes and $m$ feature nodes together 
with the edges between them.
The weight of the edge linking object $i$ and feature $j$ is defined as the 
strength or \textit{value} of feature $j$ of object $i$.
Continuous data sets, binary data sets, and vector representations of 
unstructured data can be mapped to this bipartite graph form directly;
for categorical data, the transformation is also straightforward: 
a categorical feature with $l$ different categories is replaced by a $l$ 
bits one-hot binary feature vector.
Notice that this transform does not incur any information loss.
An example of $n=4$ and $m=5$ is shown in Figure \ref{fig:bipartite}(a), where the blue nodes are objects and orange nodes are features.

To illustrate the idea behind the graph diffusion similarity, we consider 
the random walk on $\mathcal{B}$ as follows.
We initially place a particle at any object node the similarity to which
we wish to quantify, and let a random walk on the graph take place.
For example in Figure \ref{fig:bipartite}, the particle is originally placed 
at the first node.
Since there are 3 edges pointing away with weights 3, 5, and 2,
the probability vector that the particle ends up at feature nodes 1 to 5 is 
$(0.3, 0.5, 0, 0.2, 0)$, respectively, see Figure \ref{fig:bipartite}(b).
We call this a 1-step diffusion or random walk 
on $\mathcal{B}$ starting at node 1. 
Now the particle conducts another step of the random walk, 
and the corresponding probability (to arrive at any given object node) 
can be computed by adding the probabilities from all feature node. 
see Figure \ref{fig:bipartite}(c) and \ref{fig:bipartite}(d).
We call these two steps of the random walk one \textit{round}.  
The induced subgraph after one round (2 steps) gives
rise to an object-object graph that we denote by $\mathcal{G}$.
In the language of Markov chain theory, the 2$m$-step random walk
on $\mathcal{B}$ is the first $m$ rounds of iterations in the 
computation of the stationary distribution (principal eigenvector) of the 
row-normalized adjacency matrix of $\mathcal{G}$ starting at 
the localization vector $u=(0,\dots,1,\dots,0)$ where a 1 is placed at node 
1 in the above example, see \cite{norris1998markov} for details of iterative or power method.
Random walk theory is frequently used for the task of community detection~\cite{rosvall2008maps,pons2005computing}, yet to our best knowledge,
there is no previous work of adopting random walk for the construction of similarity measures.

\begin{figure}[htp!]
    \centering
           \includegraphics[trim={0 1cm 1cm 0},width=0.95\textwidth]{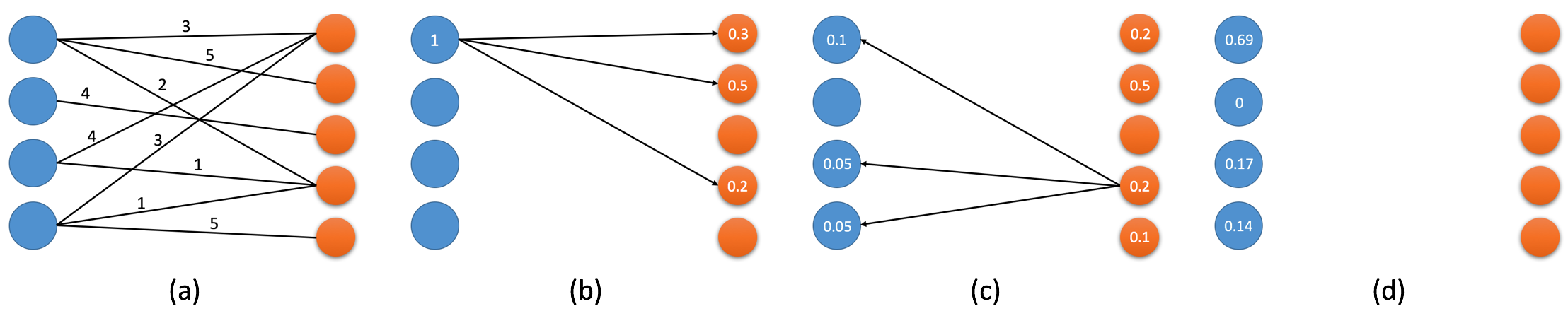}
    \caption{Example: the calculation of graph diffusion similarity. \label{fig:bipartite}}
\end{figure}

Now let the particle start at object $i$ in $\mathcal{G}$, then 
we define the order $k$ diffusion similarity of $i$ to $j$, 
denoted by $g^{(k)}(i,j)$, as the probability of the particle starting at $i$
ending up at $j$ after $k$ steps in $\mathcal{G}$, or $2k$ steps in $\mathcal{B}$.
Notice that similar objects, that is those with similar features and strengths, 
have stronger connection in the bipartite graph $\mathcal{B}$, and consequently
the $k$-round transition probability between them in $\mathcal{G}$
will be higher.
This family of similarity measures may thus be seen as a truncated and localized version of 
the principal eigenvector computation on $\mathcal{G}$.
Unlike this computation, we are interested in the finite step transition probability between 
a pair of nodes $(i,j)$ on $\mathcal{G}$ as a measure of similarity between $i$ and $j$,
which the principal eigenvector does not provide.

It is clear that $g^{(k)}(i,j)$ is not necessarily symmetric, therefore we 
define $r^{(k)}(i,j):=g^{(k)}(j,i)$, which we will refer to as the reversed 
diffusion similarity, and it quantifies the $k$-step similarity of $j$ to $i$. 
To balance the importance of each feature, one can normalize each feature vector's 
row-sum to 1 and then calculate the graph diffusion similarity.
We call the corresponding similarity, denoted by $n^{(k)}(i,j)$, 
the normalized graph diffusion similarity.
We will show later that the normalized graph diffusion similarity is 
symmetric: $n^{(k)}(i,j)=n^{(k)}(j,i)$.
All of the above are measures of similarity each with a
corresponding measure of distance: 
$g_d^{(k)}(i,j):=1-g^{(k)}(i,j)$, 
$r_d^{(k)}(i,j):=1-r^{(k)}(i,j)$, and $n_d^{(k)}(i,j):=1-n^{(k)}(i,j)$,
which are the graph diffusion distance, reversed graph diffusion distance, 
and normalized graph diffusion distance, respectively.

For simplifying the analysis, we introduce the matrix form of 
the above graph diffusion similarity and distance measures.
For the $n\times m$ feature matrix $W=(w_{ij})$ where $1\le i\le n$ and $1\le j \le m$,
define diagonal matrices $P=(p_{ij})$ and $Q=(q_{ij})$ as:
$p_{ll}=\sum_{s=1}^m w_{ls}$, $q_{ll}=\sum_{s=1}^nw_{sl}$.
In other words, $P$ and $Q$ are the row-sum and column-sum diagonal matrices 
corresponding to $W$.
We assume that $p_{ll}$ and $q_{ll}$ are non-zero, otherwise we can discard 
the null object or remove the absent feature.
When its dimension is understood, let $\bm{1}$ be the all-one column vector.
Define $n\times n$ matrix $S$ as:
\begin{equation}\label{eq:matrix}
S:=P^{-1}WQ^{-1}W^T.
\end{equation} 
Based on the definition of $P$ and $Q$, it is clear that $S$ is a row-stochastic matrix since
$S\bm{1}=P^{-1}WQ^{-1}W^T\bm{1}=P^{-1}W\bm{1}=\bm{1}$.
Recall that for the random walks on $\mathcal{B}$ or $\mathcal{G}$, 
the matrix $S$ is actually the single-step transition matrix on $\mathcal{G}$ 
or the two-step transition matrix on $\mathcal{B}$.
Let $G^{(k)}=(g^{(k)}(i,j))$ be the $n\times n$ matrix of the pairwise graph diffusion similarity, then it is clear that $G^{(1)}=S$.
The higher order diffusion similarity is straightforward to calculate:
\begin{equation}
G^{(k)}=S^k=\left(P^{-1}WQ^{-1}W^T\right)^k.
\end{equation}
In particular, for $g^{(1)}(i,j)$, an explicit formula can be written as:
\begin{equation}\label{eq:k=1}
g^{(1)}(i,j)=\sum_{s=1}^m\frac{w_{is}}{w_{i1}+\dots+w_{im}}\frac{w_{js}}{w_{1s}+\dots+w_{ns}}
=\frac{1}{p_{ii}}\sum_{s=1}^m\frac{w_{is}w_{js}}{q_{ss}}.
\end{equation}

A practical issue worth mentioning is the computational cost of the graph diffusion distance.
If the goal is to compute a single pair similarity, \eqref{eq:k=1} shows 
that the cost is $O(mn)$, which is less than ideal.
For example, the Euclidean distance or cosine similarity only requires $O(m)$ calculations. 
However, note that graph diffusion similarity for one pair of objects is not as important as
the similarity between a set of objects and a fixed object.
To wit, for similarity search and related tasks relative to an object $i$, one may need to 
calculate $g^{(1)}(i,j)$ for all $j$ followed by ranking.
From the matrix form \eqref{eq:matrix}, it is clear that the computational 
cost for this task is still $O(mn)$, 
which scales linearly in the number of objects and the number of features.
Again from \eqref{eq:matrix}, the computation of the graph diffusion similarity 
for all the pairs of objects requires $O(mn^2)$ calculations,
which is the same as other traditional similarities.
The reversed and normalized variants only involve matrix transpose and 
normalization operations, thus the cost is in the same order.
Since the calculation of the graph diffusion distance can be written 
in the matrix form \eqref{eq:matrix}, 
parallel computing is also straightforward to use if and when needed.

\section{Metametrics and Quasi-Metametrics}

In this section, we analyze the conditions under which the graph diffusion 
distance or its variants come close to a \textit{bona fide} metric.  
We assume that the bipartite graph $\mathcal{B}$ is connected,
otherwise the data set can be partitioned into groups of objects with completely different features.
As it turns out, 
we do not actually need the full
strength of a metric but only 1/2 of the key metric properties to judiciously separate points. 
In fact, some metric properties such as symmetry are even unnatural in our context:
among an object $A$'s neighbors $B$ is the most similar to $A$, 
but among the many more neighbors of $B$, $A$ may not be that similar to $B$. 
From the four key metric properties, we wish to keep non-negativity and the triangle 
inequality to preserve a notion of neighborhood.
It turns out the intersection of a quasimetric and a metametric would be 
sufficient for our purposes.  

A quasimetric is a metric without the symmetry property 
while a metametric does not require that two objects with same features to be
identical. 
In order for a function $d(\cdot,\cdot)$ to be a metametric,
$d$ has to be non-negative, symmetric, $d(x,y)=0$ to imply $x=y$ but not
necessarily vice versa, and to satisfy the triangle inequality $d(x,y)+d(y,z)\ge d(x,z)$.
Notice that in \eqref{eq:k=1}, when $p_{ii}$ is the same for all $i$, $g^{(1)}(i,j)$ becomes symmetric.
In fact, the symmetry will be inherited by $g^{(k)}(i,j)$, and we have
\begin{theorem}\label{th1}
The normalized graph diffusion distance of order $k$, namely $n_d^{(k)}(\cdot,\cdot)$, is a metametric.
When applied to distributions or categorical data, the forward, reversed, and normalized graph diffusion distances become identical, and are all metametrics as well.
\end{theorem}

A quasi-metametric is a metametric without symmetry.
A quasi-metametric captures the key asymmetric relations in object-feature data sets
and provides a good neighborhood structure.
Therefore, a quasi-metametric has the only the necessary properties for a similarity to be useful.
In the case when $p_{ii}$ are different for different $i$, symmetry no longer exists, 
and thus quasi-metametric is our best shot:

\begin{theorem}\label{th}
Let $P$ be the row-sum diagonal matrix for $W$.
If $\min p_{ii} /\max p_{ii}>2/3$, then both the forward graph diffusion 
distance $g_d^{(1)}(\cdot,\cdot)$ and te reversed graph diffusion distance 
$r_d^{(1)}(\cdot,\cdot)$ are quasi-metametrics.
\end{theorem}

Theorem \ref{th} shows that, if the row-sums of the features are comparable, then the order 1
forward graph diffusion distance and its reversed version are quasi-metametrics.
However, since the similarity vector $g^{(k)}(i,\cdot)$ eventually converges to the 
equilibrium vector of the graph $\mathcal{G}$,
the triangle inequality can not hold for large $k$ if the equilibrium vector itself 
does not follow the triangle inequality.
On the other hand, the similarity vector $r^{(k)}(i,\cdot)$ converges to a vector 
of a constant, thus  triangle inequality holds.
Detailed proofs for Theorem \ref{th1} and Theorem \ref{th} are left in the supplementary material.

\section{Experiments on Structured Data}\label{sec:structured}
In this section, we evaluate the performance of graph diffusion similarities on structured data sets.
Since similarity is used explicitly or implicitly in a wide spectrum of data science and machine learning studies, 
it is difficult to analyze and evaluate the newly proposed near-metrics thoroughly in all kinds of tasks and application.
Instead, we aim at the first-principle experiments, that we check whether two data points that are ``close'' 
according to a similarity share the same ground truth label.

More specifically, let $x$ be any chosen data point and $y$ the corresponding label.
To test the performance of a similarity measure $S$, we first rank all
the data points with respect to their similarities to $x$.
Then for any $0 < f \le 1 $, we calculate the proportion of data points that hold different 
labels compared to $y$ in the $(nf)$-nearest neighbors of $x$,
which yields the error value $e^S(x,f)$ of data point $x$ at $f$.
The error curve is defined as the averaged error for all the data points:
\begin{equation}
E^S(f):=\frac{1}{n}\sum_{x}e^S(x,f).
\end{equation}
Notice that $E^S(1)$ does not depend on the similarity measure adopted.
It is determined by the number of data points in each class.
For example, if the data set contains two classes of equal number of data points, then $E^S(1)=0.5$.
We define $E^S(0)=0$ for convenience.
Naturally, the error curve $E^S(f)$ is expected to grow (but not necessarily) when $f$ becomes larger.
In the following experiments, we did encounter cases when this curve fluctuates, 
but most of the time we found it increases monotonically.

In Table \ref{table}, we demonstrate the performance of 10 existing similarity measures 
and the graph diffusion similarity measures of order 1 to 7, denoted by GD1 to GD7.
The existing similarity measures include overlap, Eskin, IOF, OF, Lin, Goodall3, 
Goodall4, inner product, Euclidean, and cosine.
We use reversed graph diffusion during the experiments.
Recall that, when applied to categorical features, all graph diffusion similarities coincide.
Among the tested data sets, the results of 11 are shown in Table \ref{table}.
In table \ref{table}, 9 of the shown data sets are from the UCI Machine Learning Repository~\cite{asuncion2007uci}.
The LC and PR are loan level data sets from the two largest P2P sites, 
Prosper and Lending Club~\cite{club2015lending,prosper2016}.
Because of the page limit, we are not able to demonstrate the error 
curves directly in this section;
instead, we show the value of $E^S(0.01)$, $E^S(0.02)$, $E^S(0.05)$ in Table \ref{table},
which correspond to the averaged errors at $1\%$, $2\%$, and $5\%$ nearest neighbor sets.

\begin{table*}
\begin{center}
\tiny
\tabcolsep=0.15cm
\begin{tabular}{ |c|c|c|c|c|c|c|c|c|c|c|c|c|c|c|c|c|c| } 
\hline
 & overlap & Eskin & IOF & OF & Lin & G3 & G4 & inner & l2 & cosine & GD1 & GD2 & GD3 & GD4 & GD5 & GD6 & GD7 \\
\hline
Balance Scale & 0.63 & 0.63 & 0.43 & 0.63 & 0.63 & 0.63 & 0.63 & 0.63 & 0.63 & 0.63 & 0.63 & 0.39 & 0.38 & 0.37 & 0.34 & 0.32 & 0.38 \\
$m=4$      & 0.44 & 0.44 & 0.41 & 0.44 & 0.44 & 0.44 & 0.44 & 0.44 & 0.44 & 0.44 & 0.44 & 0.34 & 0.35 & 0.36 & 0.36 & 0.36 & 0.36 \\
$n=625$   & 0.63 & 0.63 & 0.41 & 0.63 & 0.63 & 0.63 & 0.63 & 0.63 & 0.63 & 0.63 & 0.63 & 0.46 & 0.44 & 0.45 & 0.44 & 0.42 & 0.44 \\
\hline
SPECT Heart & 0.28 & 0.27 & 0.23 & 0.23 & 0.23 & 0.22 & 0.28 & 0.28 & 0.28 & 0.28 & 0.21 & 0.29 & 0.34 & 0.37 & 0.39 & 0.44 & 0.43 \\
$m=22$ & 0.35 & 0.33 & 0.30 & 0.30 & 0.28 & 0.27 & 0.27 & 0.36 & 0.35 & 0.35 & 0.24 & 0.31 & 0.35 & 0.38 & 0.40 & 0.45 & 0.47 \\
$n=267$ & 0.35 & 0.33 & 0.28 & 0.28 & 0.27 & 0.25 & 0.38 & 0.35 & 0.35 & 0.35 & 0.23 & 0.28 & 0.36 & 0.40 & 0.42 & 0.43 & 0.44 \\
\hline
Lending Club  & 0.18 & 0.12 & 0.10 & 0.12 & 0.11 & 0.12 & 0.13 & 0.18 & 0.18 & 0.18 & 0.09 & 0.10 & 0.10 & 0.10 & 0.12 & 0.13 & 0.14 \\
$m=13$ & 0.24 & 0.15 & 0.15 & 0.15 & 0.15 & 0.16 & 0.15 & 0.24 & 0.24 & 0.24 & 0.14 & 0.14 & 0.15 & 0.16 & 0.16 & 0.17 & 0.16 \\
$n=39786$ & 0.24 & 0.18 & 0.17 & 0.18 & 0.18 & 0.18 & 0.18 & 0.24 & 0.24 & 0.24 & 0.17 & 0.17 & 0.17 & 0.18 & 0.19 & 0.19 & 0.19 \\
\hline
Nursery & 0.45 & 0.32 & 0.30 & 0.31 & 0.30 & 0.31 & 0.31 & 0.45 & 0.45 & 0.45 & 0.31 & 0.32 & 0.33 & 0.34 & 0.35 & 0.36 & 0.38 \\
$m=8$ & 0.52 & 0.43 & 0.35 & 0.35 & 0.36 & 0.35 & 0.38 & 0.52 & 0.52 & 0.52 & 0.37 & 0.37 & 0.38 & 0.39 & 0.41 & 0.42 & 0.43 \\
$n=12960$ & 0.60 & 0.46 & 0.43 & 0.44 & 0.43 & 0.44 & 0.45 & 0.60 & 0.60 & 0.60 & 0.44 & 0.44 & 0.45 & 0.46 & 0.47 & 0.48 & 0.49 \\
\hline
Prosper  & 0.30 & 0.21 & 0.21 & 0.20 & 0.22 & 0.23 & 0.26 & 0.30 & 0.30 & 0.30 & 0.21 & 0.22 & 0.25 & 0.29 & 0.32 & 0.35 & 0.40 \\
$m=18$ & 0.46 & 0.32 & 0.31 & 0.32 & 0.33 & 0.33 & 0.37 & 0.48 & 0.48 & 0.48 & 0.36 & 0.33 & 0.34 & 0.34 & 0.36 & 0.36 & 0.38 \\
$n=58845$ & 0.52 & 0.42 & 0.40 & 0.39 & 0.40 & 0.40 & 0.44 & 0.52 & 0.52 & 0.52 & 0.40 & 0.40 & 0.39 & 0.40 & 0.39 & 0.40 & 0.40 \\
\hline
Tic-Tac-Toe & 0.44 & 0.44 & 0.13 & 0.31 & 0.18 & 0.21 & 0.12 & 0.44 & 0.44 & 0.44 & 0.21 & 0.20 & 0.17 & 0.15 & 0.15 & 0.15 & 0.16 \\
$m=9$& 0.23 & 0.23 & 0.19 & 0.23 & 0.20 & 0.22 & 0.21 & 0.23 & 0.23 & 0.23 & 0.25 & 0.22 & 0.20 & 0.19 & 0.19 & 0.20 & 0.21 \\
$n=958$ & 0.45 & 0.45 & 0.25 & 0.39 & 0.34 & 0.36 & 0.26 & 0.45 & 0.45 & 0.45 & 0.36 & 0.34 & 0.32 & 0.30 & 0.29 & 0.29 & 0.29 \\
\hline
Car Evaluation  & 0.29 & 0.23 & 0.21 & 0.23 & 0.22 & 0.24 & 0.22 & 0.29 & 0.29 & 0.29 & 0.24 & 0.23 & 0.23 & 0.23 & 0.23 & 0.24 & 0.24 \\
$m=6$ & 0.52 & 0.29 & 0.23 & 0.30 & 0.27 & 0.30 & 0.24 & 0.52 & 0.52 & 0.52 & 0.30 & 0.31 & 0.31 & 0.31 & 0.31 & 0.30 & 0.30 \\
$n=1728$ & 0.40 & 0.32 & 0.31 & 0.33 & 0.31 & 0.34 & 0.31 & 0.40 & 0.40 & 0.40 & 0.34 & 0.34 & 0.34 & 0.33 & 0.33 & 0.34 & 0.34 \\
\hline
Chess  & 0.30 & 0.27 & 0.20 & 0.20 & 0.19 & 0.19 & 0.24 & 0.30 & 0.30 & 0.30 & 0.20 & 0.26 & 0.34 & 0.39 & 0.41 & 0.43 & 0.45 \\
$m=36$ & 0.36 & 0.33 & 0.26 & 0.26 & 0.25 & 0.26 & 0.30 & 0.36 & 0.36 & 0.36 & 0.26 & 0.30 & 0.36 & 0.40 & 0.42 & 0.44 & 0.45 \\
$n=3196$ & 0.42 & 0.38 & 0.33 & 0.33 & 0.32 & 0.32 & 0.36 & 0.42 & 0.42 & 0.42 & 0.32 & 0.35 & 0.40 & 0.43 & 0.44 & 0.46 & 0.47 \\
\hline
Hayes-Roth & 0.23 & 0.23 & 0.23 & 0.23 & 0.23 & 0.23 & 0.23 & 0.23 & 0.23 & 0.23 & 0.22 & 0.20 & 0.65 & 0.71 & 0.71 & 0.82 & 0.81 \\
$m=5$ & 0.32 & 0.32 & 0.32 & 0.25 & 0.25 & 0.24 & 0.30 & 0.32 & 0.32 & 0.32 & 0.22 & 0.21 & 0.60 & 0.61 & 0.63 & 0.68 & 0.74 \\
$n=160$ & 0.56 & 0.56 & 0.52 & 0.47 & 0.49 & 0.46 & 0.36 & 0.56 & 0.56 & 0.56 & 0.38 & 0.31 & 0.34 & 0.42 & 0.50 & 0.53 & 0.58 \\
\hline
Connect-4 & 0.50 & 0.49 & 0.41 & 0.40 & 0.42 & 0.40 & 0.45 & 0.50 & 0.50 & 0.50 & 0.42 & 0.45 & 0.49 & 0.51 & 0.51 & 0.51 & 0.51 \\
$m=42$ & 0.54 & 0.51 & 0.45 & 0.43 & 0.46 & 0.44 & 0.47 & 0.54 & 0.54 & 0.54 & 0.46 & 0.47 & 0.50 & 0.50 & 0.51 & 0.51 & 0.51 \\
$n=67557$ & 0.55 & 0.52 & 0.47 & 0.46 & 0.48 & 0.47 & 0.49 & 0.55 & 0.55 & 0.55 & 0.48 & 0.49 & 0.50 & 0.51 & 0.51 & 0.51 & 0.51 \\
\hline
Solar Flare  & 0.42 & 0.39 & 0.36 & 0.35 & 0.35 & 0.35 & 0.38 & 0.42 & 0.42 & 0.42 & 0.35 & 0.38 & 0.41 & 0.38 & 0.35 & 0.33 & 0.34 \\
$m=10$ & 0.38 & 0.34 & 0.31 & 0.30 & 0.30 & 0.31 & 0.33 & 0.38 & 0.38 & 0.38 & 0.31 & 0.32 & 0.35 & 0.37 & 0.38 & 0.38 & 0.38 \\
$n=1389$ & 0.36 & 0.33 & 0.28 & 0.29 & 0.28 & 0.29 & 0.31 & 0.36 & 0.36 & 0.36 & 0.29 & 0.29 & 0.31 & 0.32 & 0.33 & 0.33 & 0.33 \\
\hline
\end{tabular}
\end{center}
\caption{Evaluation of similarity measures on different data sets.
For each similarity $S$, the three rows describe the value 
of $E^S(0.01)$, $E^S(0.02)$, and $E^S(0.05)$.
For each data set, its object number $n$ and attribute number $m$ are listed below  
its name.
\label{table}}
\end{table*}

We make several key observations as follows.
It is shown in Table \ref{table} that no single similarity measure dominates all others, 
which is in accordance with the observation in \cite{boriah2008similarity} and the {\it no free lunch theorem} 
in optimization and machine learning~\cite{wolpert2002supervised}.
The order 1 forward graph diffusion similarity $g^{(1)}(\cdot,\cdot)$ is 
among the best, while IOF, Lin, and Goodall3 also perform well on certain data sets.
In addition, it can be observed that $g^{(1)}(\cdot,\cdot)$ usually 
performs the best compared to its higher order versions.
We will come back to this observation in the next section.

\section{Experiments on Unstructured Data}\label{sec:unstructured}

In this section, we present experimental evaluations of similarity 
measures on vector representations of unstructured data.
Unstructured data like text and images, despite their vast volume and importance, 
are less understood and more difficult to analyze.
The major reason appears to be that unstructured data can not be used directly
for the tasks of search, clustering, or recommendation.
To this end, various ways of transforming unstructured data into the vector form 
have been proposed~\cite{mikolov2013distributed,le2014distributed}.
For example, tf-idf converts a sentence or a document into a sparse, high-dimensional vector~\cite{salton1975vector},
and word2vec model translates texts into dense vector with semantics in low dimensions~\cite{mikolov2013distributed}.
Similar approaches for data sets consisting of or including images are still awaiting detailed documentation.

In this section, we focus on the experimental evaluation of similarity measures based
on vector representations derived from tf-idf and deep learning.
Our evaluation criterion remains the same as in Section \ref{sec:structured}: 
to use the error curve $E^{S}(f)$ for comparitson.
The experiments are conducted on the IMDB movie review dataset~\cite{maas-EtAl:2011:ACL-HLT2011}, 
the customer verbatim dataset, 
and the ImageNet dataset~\cite{deng2009imagenet}.

\subsection{IMDB Movie Review Dataset} 
The IMDB dataset consists of 50,000 movie reviews in the text 
form~\cite{maas-EtAl:2011:ACL-HLT2011}.
The length of the review varies from very short to more than 2,000 words.
Each movie review is associated with a binary sentiment polarity label.
There are 25,000 positive reviews and 25,000 negative reviews.
Naturally, a good similarity measure should yield a higher similarity score 
for reviews holding the same label. We test the similarity measures on two kinds of representations:
the tf-idf representation and a compact embedding from a convolutional neural network.

\textbf{tf-idf representation.} 
A direct way of translating a review into a feature vector is via the tf-idf 
representation~\cite{salton1975vector}.
In general, the importance of a word in a review is addressed by its frequency 
in that review multiplied by its inverse document frequency in the entire corpus.
Even though this transformation risks over-simplification, it captures some basic 
frequency-related information of the texts and is widely adopted~\cite{ramos2003using,berger2000bridging}.

In the left of Figure \ref{fig:imdb_tfidf}, we demonstrate the error curves of the 
forward graph diffusion similarity, its reversed and normalized variants,
and several traditionally used similarity measures.
It can be observed that the traditional measures including Euclidean, 
Manhattan, inner product, and cosine, are considerably outperformed by 
the reversed and the normalized graph diffusion similarities.

\begin{figure}[htp!]
    \centering
    \begin{tabular}{cc}
       \includegraphics[width=0.45\textwidth]{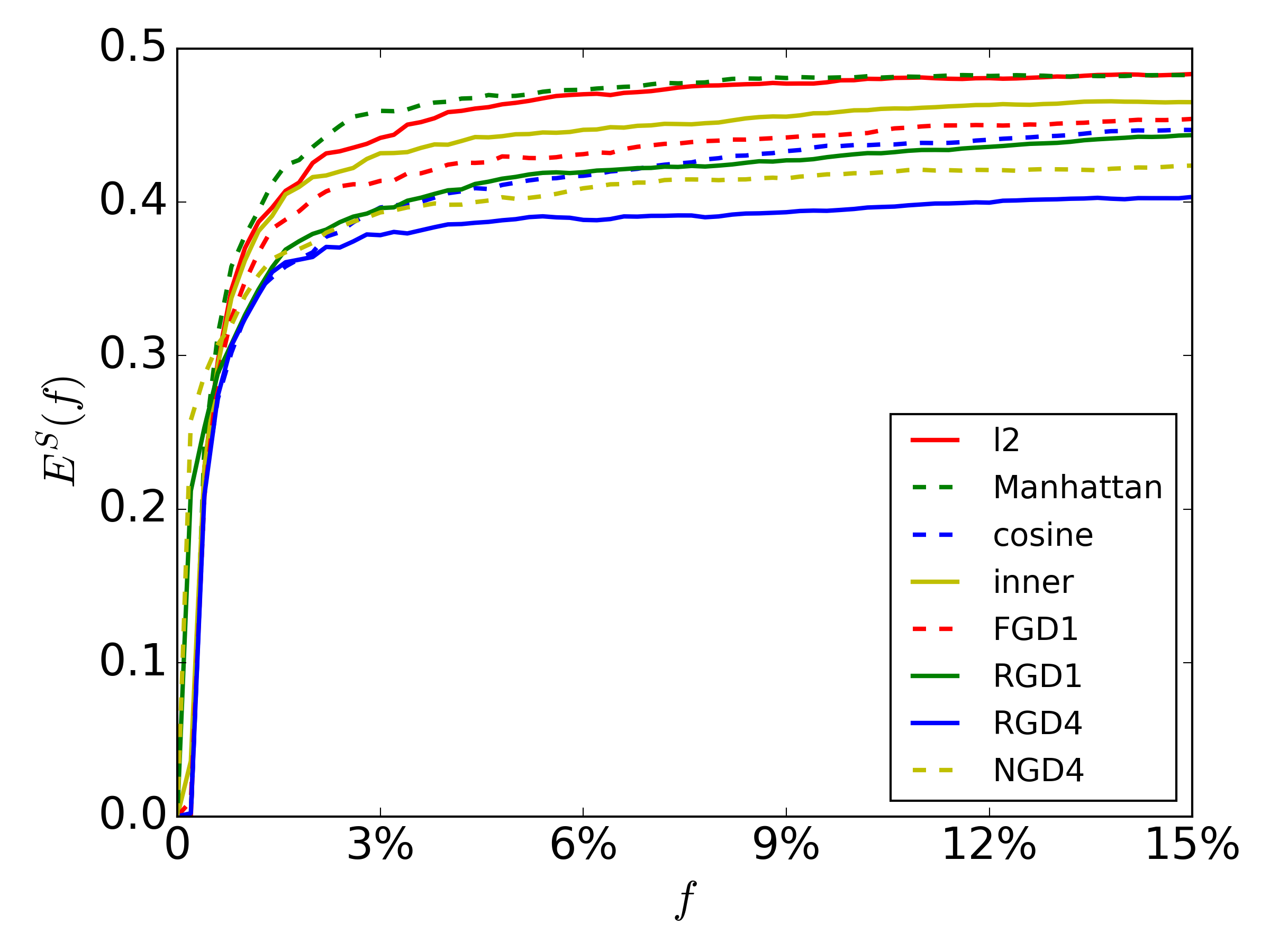}\hfill
 \includegraphics[width=0.45\textwidth]{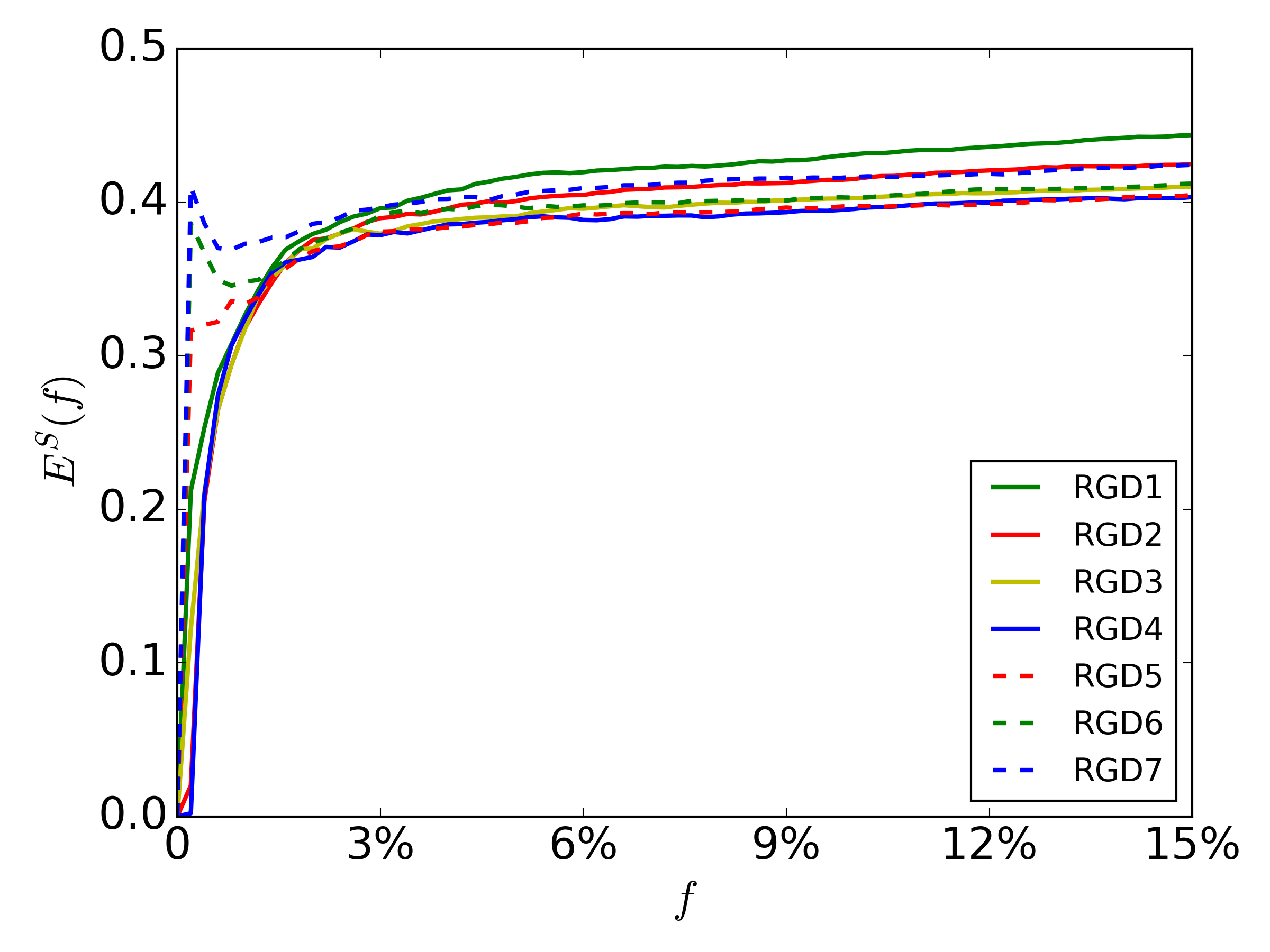} \\
    \end{tabular}
    \caption{Error curves for  the tf-idf representation of IMDB movie review dataset. 
       \label{fig:imdb_tfidf}} 
\end{figure}

The family of graph diffusion similarities also perform differently.
We demonstrate in the right of Figure \ref{fig:imdb_tfidf} the reversed 
graph diffusion similarity measures of order from 1 to 7 as an example.
It can be observed that the performances improves and later decreases as 
the order increases, and the order 4 and order 5 curves are the best among the 7.

\textbf{Compact embedding via deep learning.}
In addition to the tf-idf transformation, a more recent way of embedding a 
review into vector space is via deep learning.
The deep convolutional neural network we use here is designed for sentiment 
analysis by Kim~\cite{kim2014convolutional}.
The input layer converts any incoming paragraph into a vector of undetermined 
length,
then different sizes of convolutional windows further transform the vector 
into vectors of values.
After that, a max pooling layer eliminate the varying length and thus the 
number of nodes is the same as the number of convolution windows.
After the training session, the part of the neural network from the input 
layer to the last 
hidden layer itself becomes a function that maps any paragraph of texts 
into a fixed length of vector.

In the following experiments, the vectors are from a CNN with 128 convolution 
kernels and thus the feature vector consists of 128 dimensions.
Similar results are observed for different sizes of CNN as long as the number 
is not too small to lose track of the original information in the sentences. 
In Figure \ref{fig:IMDB}, we show the optimal error curve for reference.
The optimal error curve corresponds to the optimal similarity measure under which each review's neighbors are always with the same opinion label and the reviews holding different 
labels are far away from each other. 
Therefore the first half of the optimal error curve is 0 and then it gradually increases to 0.5.
It can be observed at the left of Figure \ref{fig:IMDB} that the reversed and normalized graph diffusion similarity measures outperform others in a clear way.
Furthermore, it is shown in the right of Figure \ref{fig:IMDB} that the order 2 normalized graph diffusion similarity almost coincide with the optimal curve.
\begin{figure}[htp!]
    \centering
    \begin{tabular}{cc}
       \includegraphics[width=0.45\textwidth]{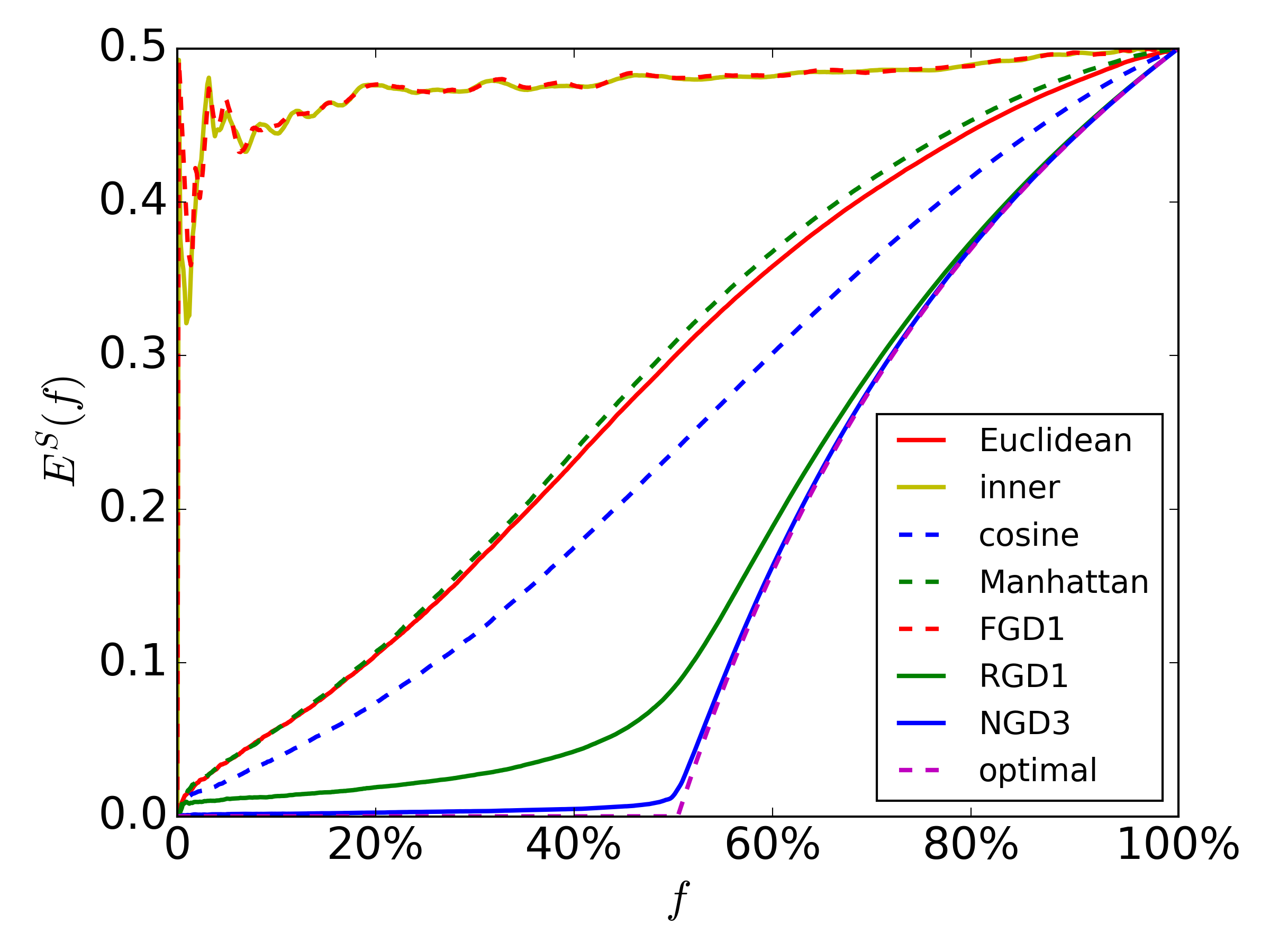}\hfill
 \includegraphics[width=0.45\textwidth]{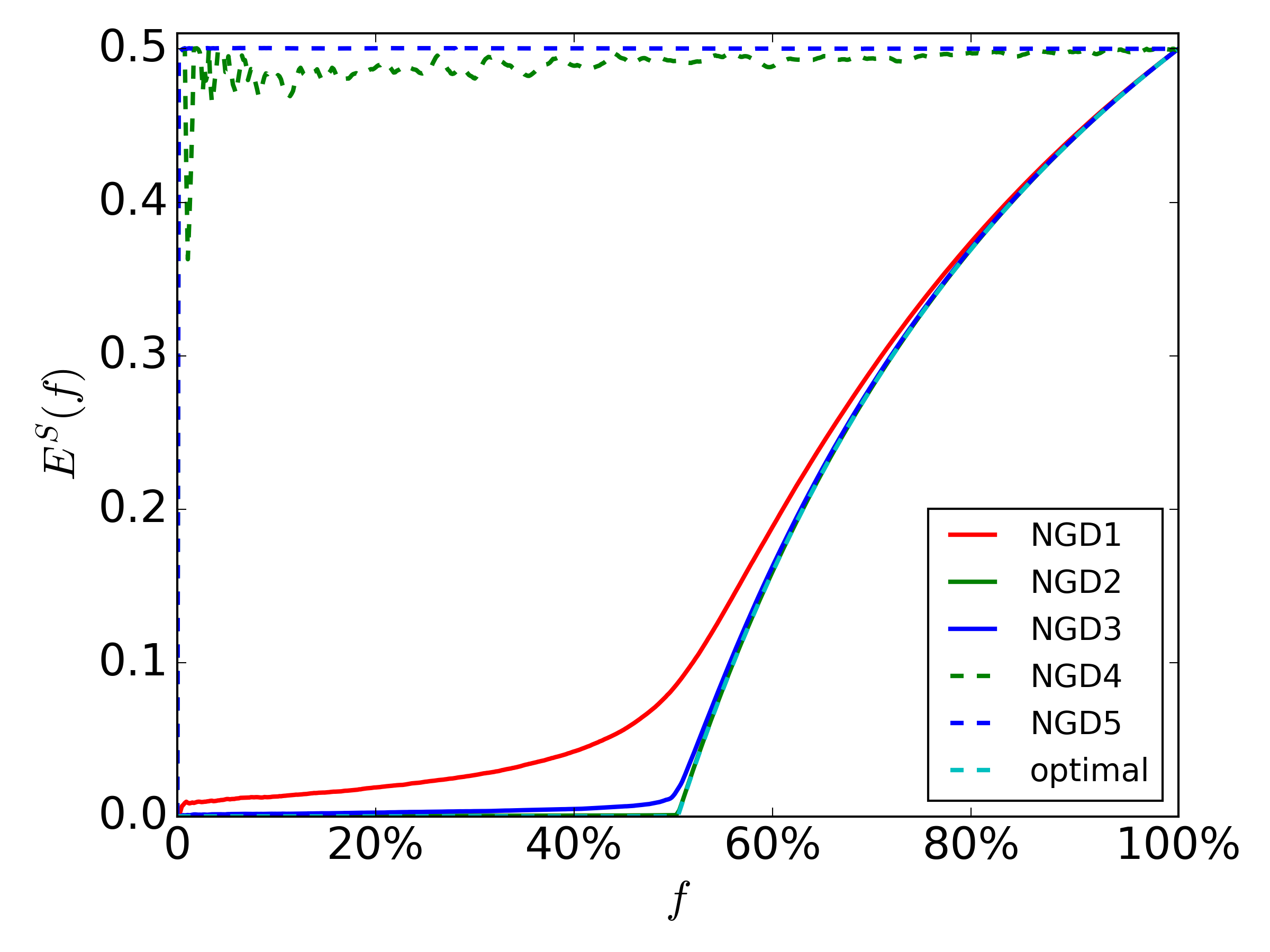} \\
    \end{tabular}
    \caption{Error curves for the compact representation of IMDB movie review dataset. 
    \label{fig:IMDB}} 
\end{figure}

An interesting phenomenon regarding graph diffusion similarities of different orders 
emerges from Figure \ref{fig:imdb_tfidf} and Figure \ref{fig:IMDB}:
when the order increases, the performance increases at first (first phase) up to
a critical order but then decreases later to 
become completely random (second phase).
This second phase is natural because the graph diffusion similarity $d^{(k)}(i,\cdot)$ 
converges to the equilibrium vector of the graph,
and the reversed and normalized versions converge to vector of a constant, 
all of which are doomed to be poor.
As for the first phase and its critical order, for the compact representation 
in Figure \ref{fig:IMDB}, the best 
performance is achieved at order 2, whereas for the sparse representation of tf-idf,
the optimal order is at 4 or 5.
From the discussion in Section \ref{sec:formula} we recall that the speed 
of the information propagation in the graph is highly related to 
the sparsity of the feature matrix $W$.
When $W$ or the graph is too sparse, the similarity between a pair of objects can 
hardly be adequately quantified if the initial probability mass is not well 
diffused in the entire system.
Thus, we need higher order to achieve good performance for sparse data.

\subsection{Customer Verbatim Data set} 
The customer verbatim data set is a private data set containing 21,621 paragraphs, 
each being a review of a company and its competitors' products and services
quantified by a collection of business customers.
Each of the reviews is labeled according to whether it is in favor of the company's competitors, 
and whether the customer likes the company's main product.
Again, we first use a CNN to do an embedding for extraction of compact feature vectors.
Eventually, each review is represented by a 128-dimensional vector of 
non-negative values.
We calculate the error curves for both the opinion-of-competitor label and the 
opinion-of-product label.
The curves are shown in Figure \ref{fig:verbatim}.

\begin{figure}[htp!]
    \centering
    \begin{tabular}{cc}
       \includegraphics[width=0.45\textwidth,trim={0 0 0 10cm}]{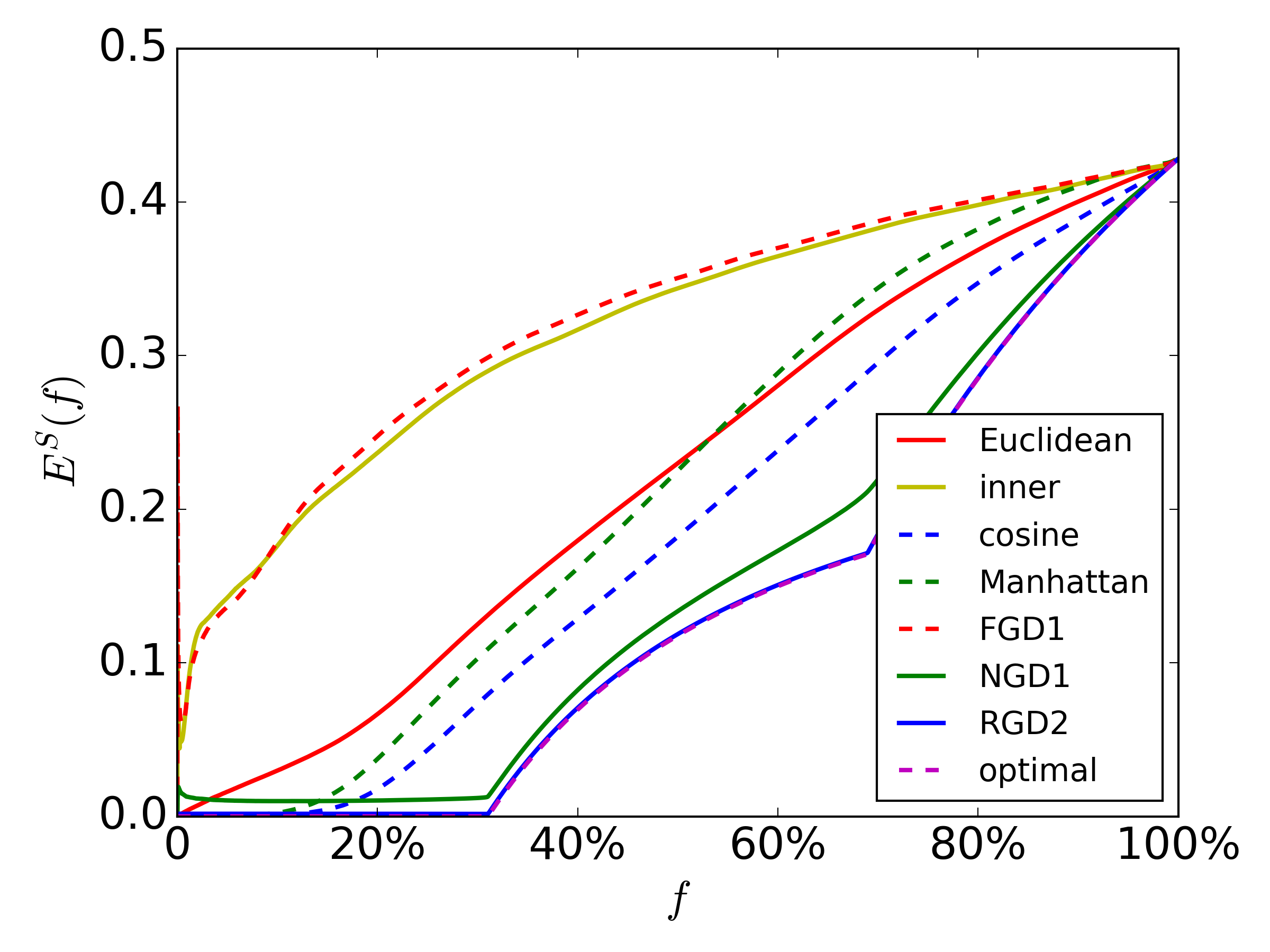}\hfill
 \includegraphics[width=0.45\textwidth]{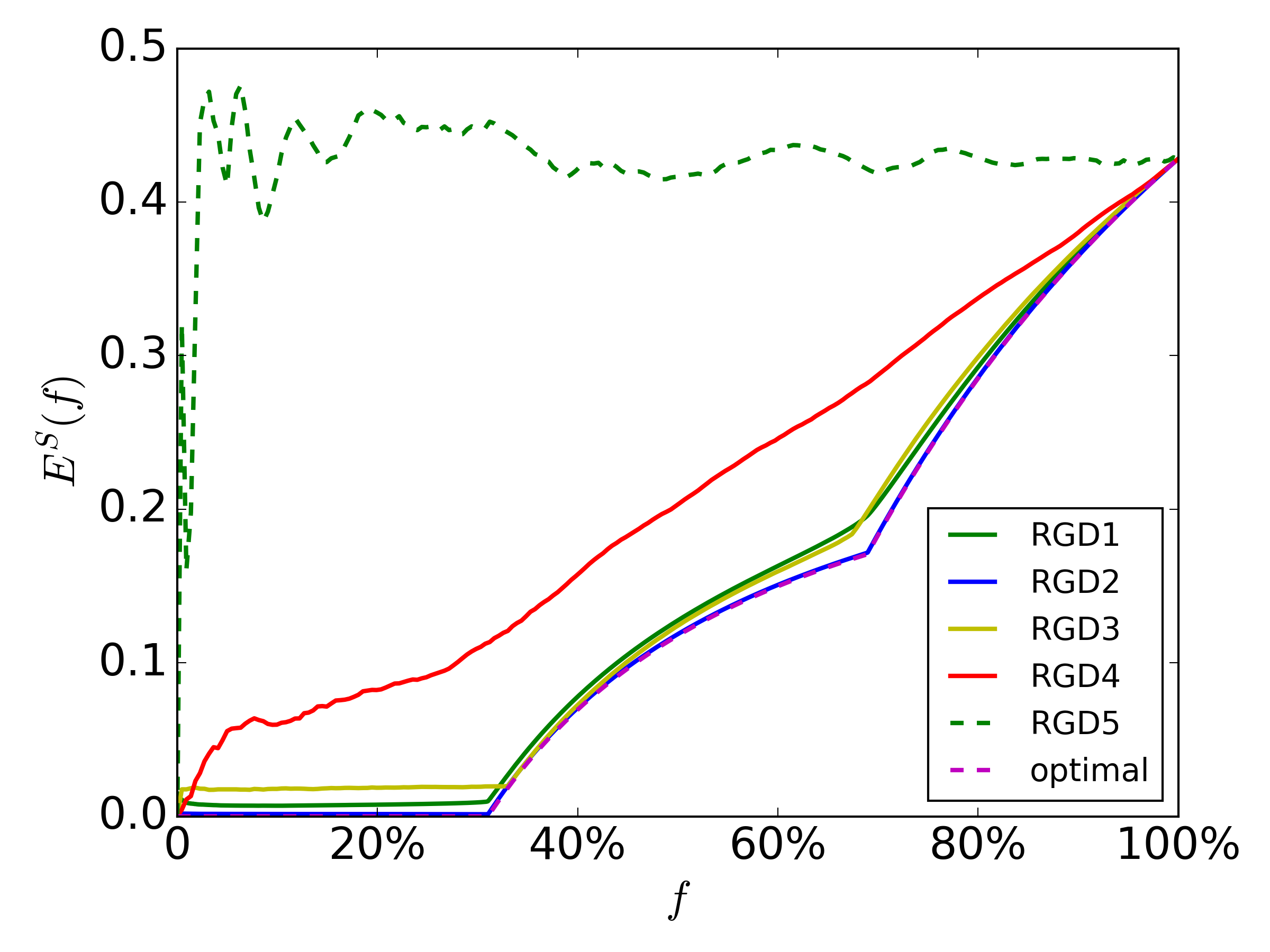} \\
       \includegraphics[width=0.45\textwidth,trim={0 0 0 10cm}]{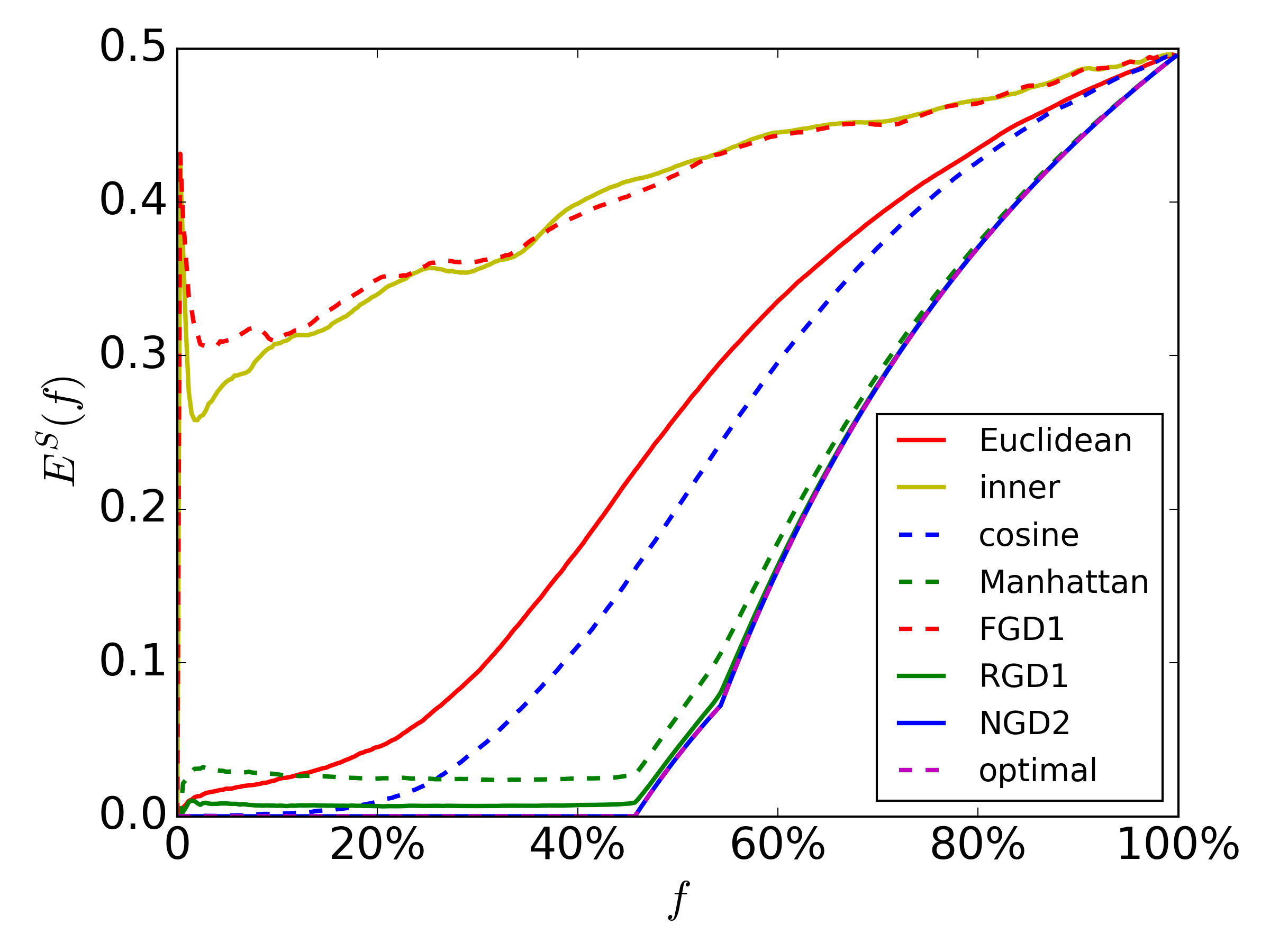}\hfill
 \includegraphics[width=0.45\textwidth]{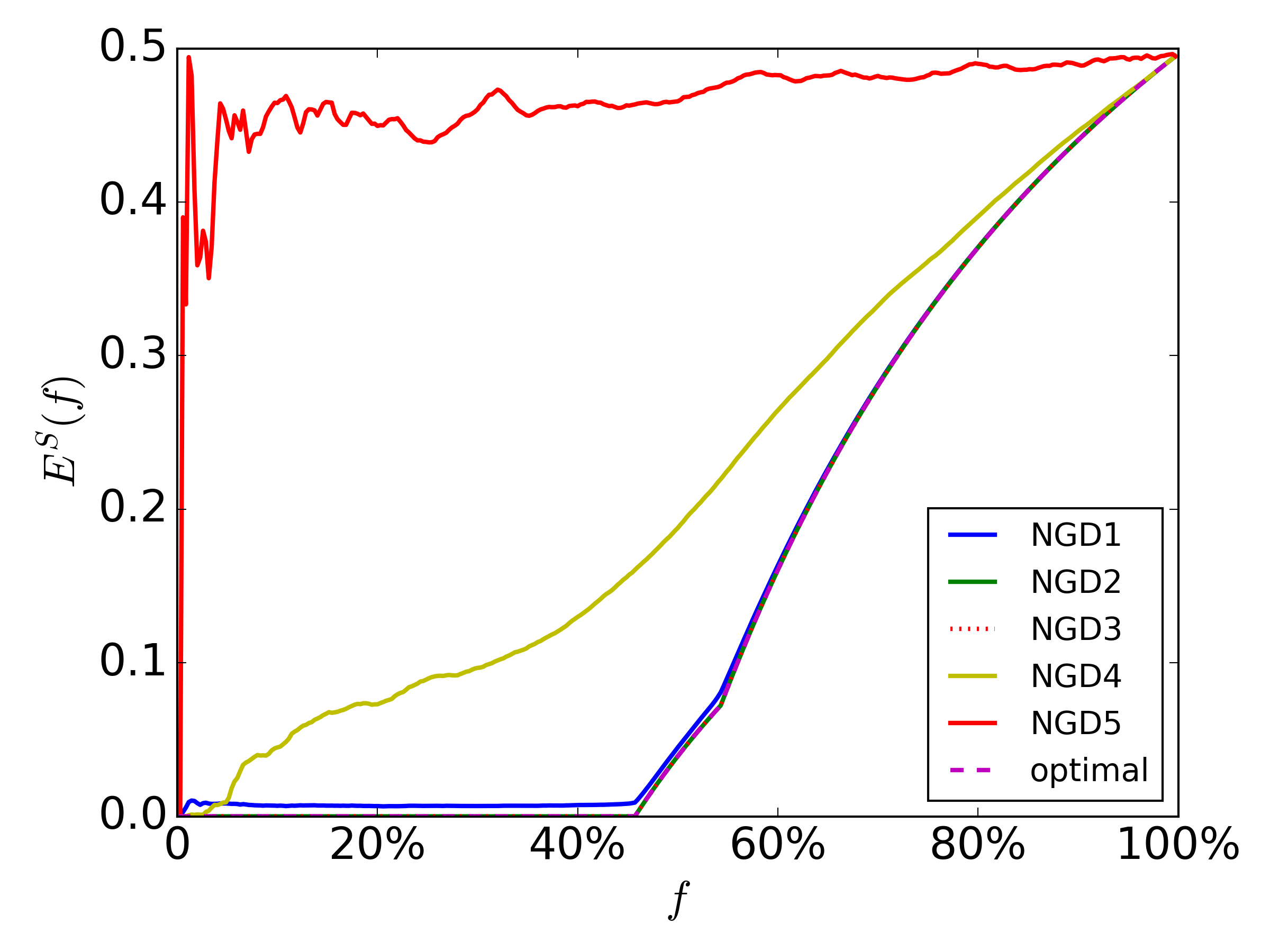} 
    \end{tabular}
    \caption{(Left) Error curves for the competitor label; (right) error curves for the product label.
    \label{fig:verbatim}} 
\end{figure}

We observe that the graph diffusion similarity family is able to achieve near 
optimal performance, 
while traditional similarities only perform in a mediocre way.
As before, the performance of the graph diffusion similarity family improves to reach a 
near optimal state 
(at order 2 for opinion-of-competitor and order 2 or 3 for opinion-of-product) and 
then deteriorates.
This again supports our conjecture that the optimal order for the graph diffusion 
similarity is positively correlated to the feature sparsity.

\subsection{ImageNet Dataset}
The performance of the graph diffusion similarities are also tested on the ImageNet dataset~\cite{deng2009imagenet}.
We adopt the GoogLeNet model trained on 1.2 million images for a classification problem for 1,000 classes~\cite{szegedy2015going}.
Again, we extract the last hidden layer, which contains 1,024 hidden nodes of non-negative values for each image.
For better visualization, we randomly pick $r$ classes out of the 1,000 different classes and calculate the corresponding error curves.
We repeat this for 200 times and compute the averaged error curves.
The case $r=2$ and $r=5$ are demonstrated in Figure \ref{fig:image}.

\begin{figure}[htp!]
    \centering
    \begin{tabular}{cc}
       \includegraphics[width=0.45\textwidth]{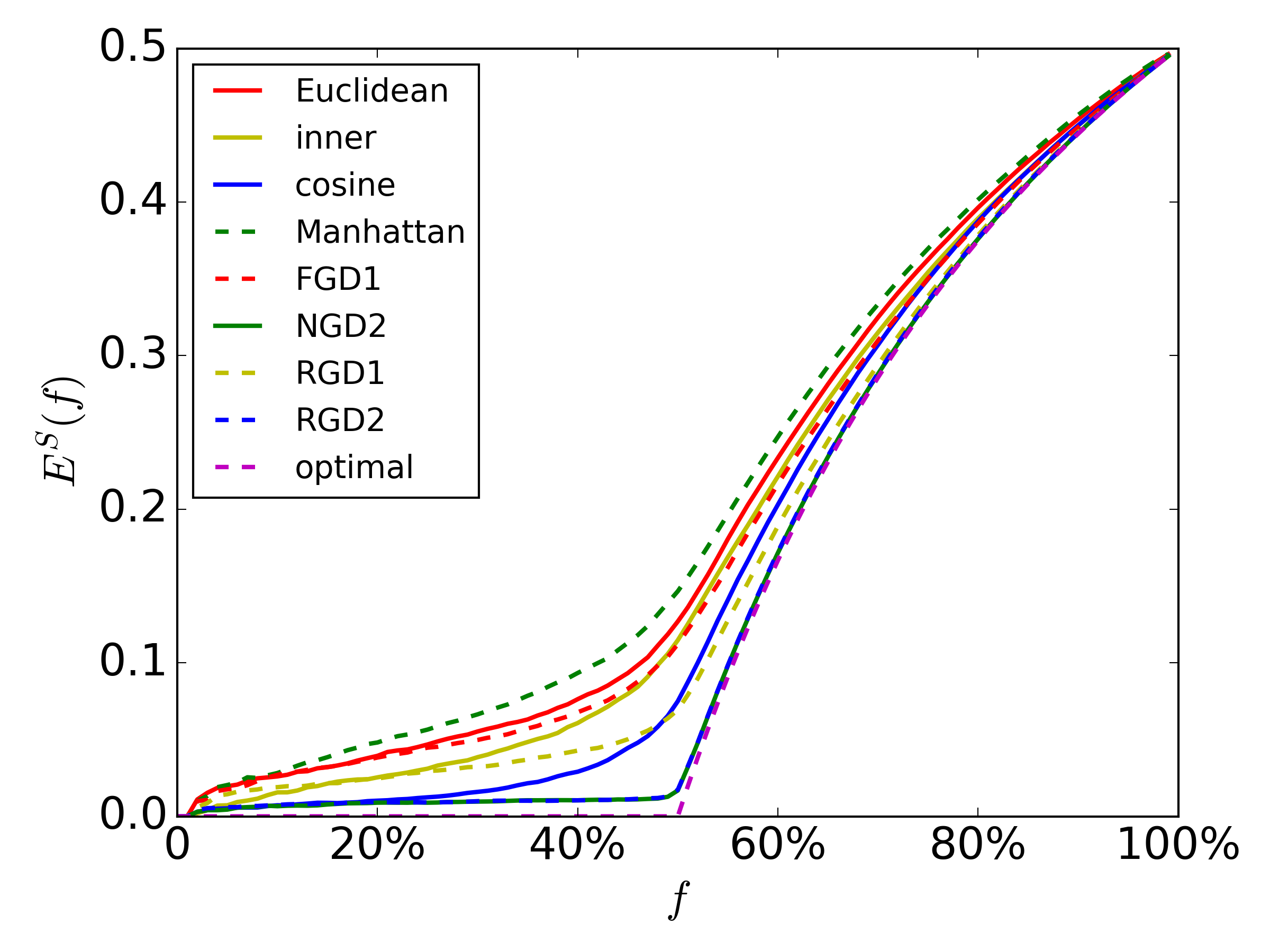}\hfill
 \includegraphics[width=0.45\textwidth]{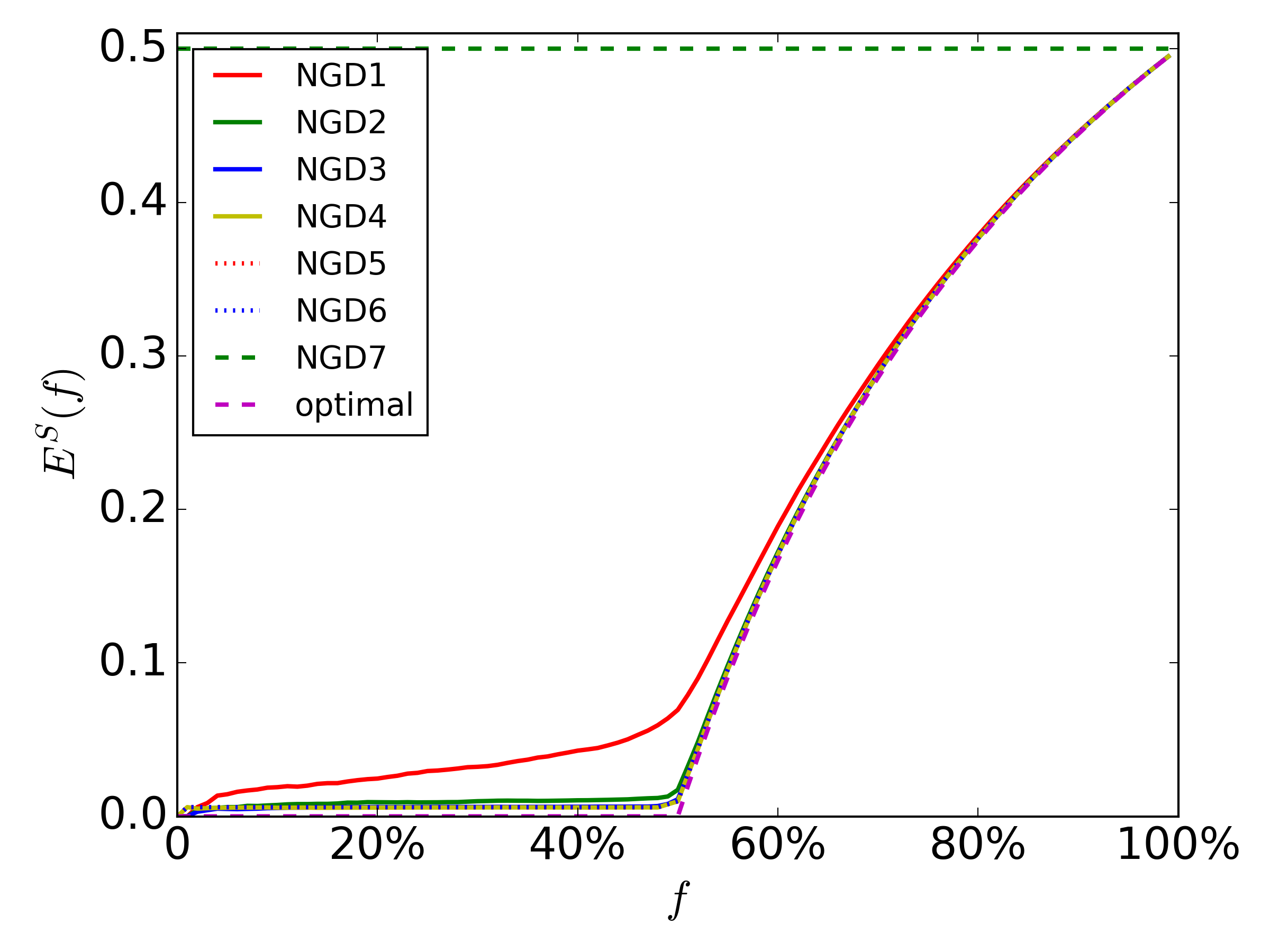} \\
       \includegraphics[width=0.45\textwidth]{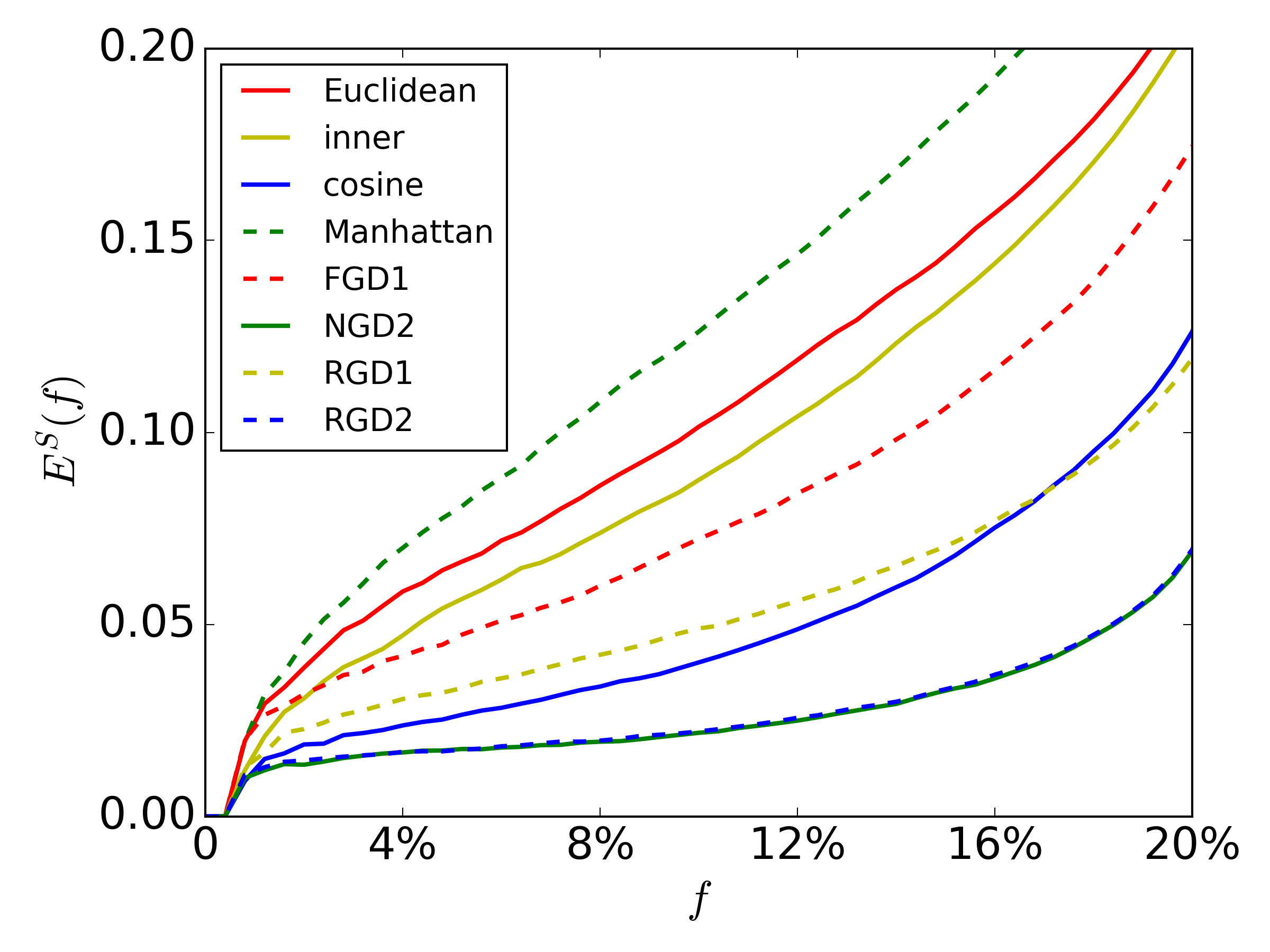}\hfill
 \includegraphics[width=0.45\textwidth]{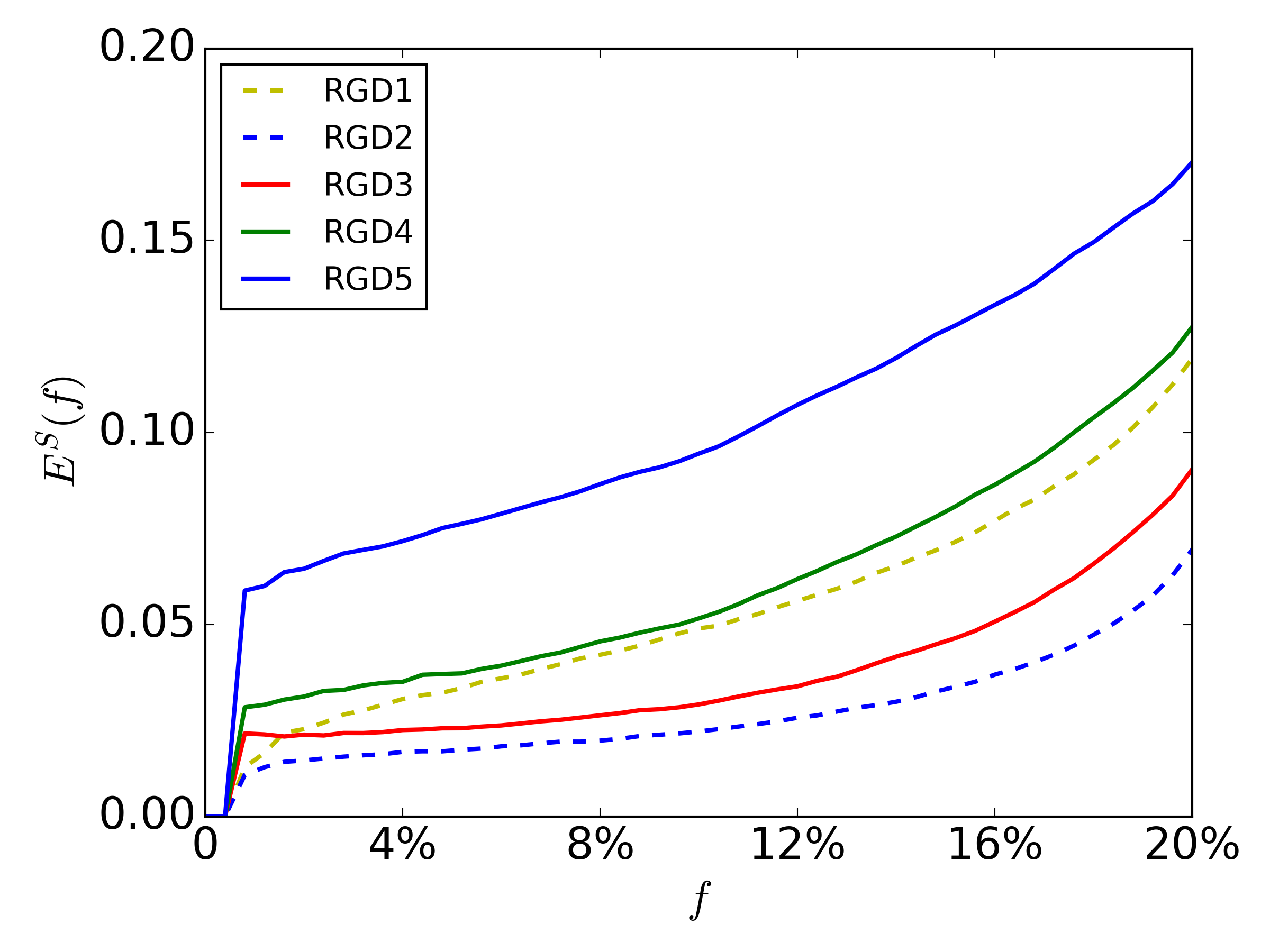} 
    \end{tabular}
    \caption{Error curves of (upper) $r=2$ and (lower) $r=5$ for GoogleNet dataset.
 \label{fig:image}}
\end{figure}

Again, compared to traditional similarities, the graph diffusion similarity family performs
better in terms of distinguishing different classes. 
For the case of $r=5$, the order 2 curve performs the best, while for the case of $r=2$, 
we observe different behaviors.
The performances of the curve for order 2 to 6 are almost identical, and suddenly the 
order 7 curve becomes the constant 0.5,
which indicates the similarity has homogenized for any pair of objects.

\section{Discussion and Future Work}\label{sec:conclusion}
A key observation in the last section was that when increasing the 
number of iterations, the performance of the graph diffusion similarity usually 
improves at first and then becomes worse.
The optimal number of iterations for sparse representations like the tf-idf is typically larger (4 to 5) and for dense embeddings it is smaller (2).
For structured data sets in Section \ref{sec:structured}, the same phenomenon is observed.
We have observed similar phenomenon in other data sets which, due to the page limit, 
 we do not demonstrated here.
Purely from the perspective of graph theory, object-feature sparsity means that it takes 
more iterations for mass to diffuse from one object to another,
which makes graph diffusion similarity less ideal with small iterations.
We conjecture that this phenomenon should be common for all kinds of 
data sets, and call for future work on building an explicit relation 
between the sparsity of the feature vectors and the optimal number of iterations.


\newpage
\bibliography{bib} 
\bibliographystyle{plain}
\newpage
\section*{Appendix: proofs of Theorem 1 and Theorem 2}
The detailed proofs for Theorem 1 and Theorem 2 are shown as follows.
Recall that an explicit formula for $g^{(1)}(i,j)$ is written as
\begin{equation}\label{eq:k=1}
g^{(1)}(i,j)=\sum_{s=1}^m\frac{w_{is}}{w_{i1}+\dots+w_{im}}\frac{w_{js}}{w_{1s}+\dots+w_{ns}}
=\frac{1}{p_{ii}}\sum_{s=1}^m\frac{w_{is}w_{js}}{q_{ss}}.
\end{equation}

\subsection*{Basic Properties}\label{sec:basic}
We begin with some basic properties of the graph 
diffusion similarity and its variants.
It is clear that $0\le g_d^{(k)}(i,j)=1- g^{(k)}(i,j)\le1$ since $g^{(k)}(i,j)$ 
is a (transition) probability or transferred mass.
Further, we have
\begin{proposition}\label{prop2}
If $g_d^{(k)}(i,j)=0$, then $i=j$ and object $i$ is isolated from the rest of the objects.
If $g_d^{(k)}(i,j)=1$ for any $k$, then object $j$ 
can not be reached by object $i$ in $\mathcal{G}$.
\end{proposition}

\begin{proof}
If $g_d^{(k)}(i,j)=0$, then $g^{(k)}(i,j)=1$, then all the probability 
measure at $i$ is transferred to $j$ after $k$ rounds. 
If $i\neq j$, then there should be a feature $s$ such that $w_{is}>0$ in order 
for the probability starting at $i$ to propagate to $j$ via a path.
However, $w_{is}>0$ means that $i$ is also connected to itself and thus 
there will always be a positive probability at $i$, which implies $g_d^{(k)}(i,j)>0$.
Thus, by contradiction, $i=j$ and for any $s$ such that $w_{is}>0$, $w_{ls}=0$ for $l\neq i$, 
hence $i$ is isolated.
On the other hand, if $g_d^{(k)}(i,j)=1$ for all $k$, then there is zero 
probability transferred from $i$ to $j$ in any steps,
therefore there is no path from $i$ to $j$ in $\mathcal{G}$.
\end{proof}

Notice that $g_d^{(k)}(i,i)$ is not necessarily 0, due to dispersion of mass to other nodes through common features.
As for symmetry in the graph diffusion similarity, it is clear from \eqref{eq:k=1} 
that $g_d^{(1)}(i,j)$ is not symmetric in general.
However, we have
\begin{proposition}\label{prop:symmetry}
A sufficient condition for $G^{(1)}$ to be symmetric is that $p_{ii}$ are the same for all $i$.
If the bipartite graph is connected, then the condition that all the $p_{ii}$ are the same is also necessary for $G^{(1)}$ to be symmetric.
\end{proposition}
\begin{proof}
The first part is straightforward by \eqref{eq:k=1}.
For the second part of the statement, notice that for any pair $(i,j)$, there should be a connected path $i,k,\dots,j$,
and $i$ being connected to $k$ means $w_{i1}w_{k1}+\dots+w_{im}w_{km}>0$ and thus $p_{ii}=p_{kk}$ according to \eqref{eq:k=1}.
The same argument holds for any consecutive objects in the path between $i$ and $j$, which leads to the conclusion that $p_{ii}=p_{jj}$.
\end{proof}

\subsection*{The Triangle Inequality}
In this section, we assume $p_{ii}$ is a constant for all $i$ which could be
achieved via scaling. 
We will show that this condition ensures the resulting graph diffusion distance will be a
metametric.
It is clear that the normalized graph diffusion distance satisfies this condition, since it 
normalizes the feature weights before calculating similarity.
Further, for distributional data and categorical data this condition always holds since
for distributions $p_{ii}=1$, and for categorical data, $p_{ii}$ equals the number of categories.
Therefore, the forward, reversed, and normalized variants of the graph diffusion distance 
are identical when applied to distributions or categorical data.
In the following analysis, we will use $n_d^{(k)}(\cdot,\cdot)$ for concreteness, 
and the proof for distributions and categorical data directly follows.

We first consider the case when the number of iterations is 1. In such a case, we will prove:
\begin{proposition}\label{prop1}
For any row-stochastic matrix $W$ and its column-sum diagonal matrix $Q$, 
define matrix $D=(d_{ij})$ as $D:=\bm{1}\bm{1}^T-WQ^{-1}W^{T}$.
Then $D$ is a symmetric matrix, and the triangular inequality 
$d_{ij}+d_{jk}\ge d_{ik}$ holds for any $1\le i,j,k\le n$.
\end{proposition}

\begin{proof}
Based on the definition of $D$, we have
\begin{equation}
d_{ij}=1-\sum_{k=1}^m\frac{w_{ik}w_{jk}}{q_{kk}},
\end{equation}
and thus the symmetry of $D$ follows immediately.

For the triangle inequality, without loss of generality, we only need to prove that 
$d_{12}+d_{23}\ge d_{13}$, which is expanded as
\begin{equation}\label{eq:tri}
\sum_{k=1}^m\frac{(w_{1k}+w_{3k})w_{2k}-w_{1k}w_{3k}}{q_{kk}}\le 1.
\end{equation}
Notice that $w_{1k}w_{3k}\ge 0$ and $q_{kk}\ge w_{1k}+w_{2k}+w_{3k}$, then it is sufficient to prove that
$$\sum_{k=1}^m\frac{(w_{1k}+w_{3k})w_{2k}}{w_{1k}+w_{2k}+w_{3k}}\le 1.$$
It is easy to check $$\frac{xy}{x+y}\le\frac{1}{9}x+\frac{4}{9}y$$ holds for any $x$ and $y$ given $x+y>0$,
since it is equivalent to $(x-2y)^2\ge 0$.
By letting $x=w_{1k}+w_{3k}$ and $y=w_{2k}$, we have
$$\sum_{k=1}^m \frac{(w_{1k}+w_{3k})w_{2k}}{w_{1k}+w_{2k}+w_{3k}}\le \frac{1}{9}\sum_{k=1}^m(w_{1k}+w_{3k})+\frac{4}{9}\sum_{k=1}^mw_{2k}=\frac{2}{3}<1,$$
which completes the proof of triangle inequality, and thus the proposition. 
\end{proof}

We note that the coefficient $2/3$ in the above is tight in the sense 
that there exists a construction of $W$ such that all the above inequalities 
become equality. 
The construction is as follows. 
\begin{itemize}
\item
Let $w_{11}=1$, $w_{12}=0$, $w_{21}=w_{22}=1/2$, $w_{31}=0$, $w_{32}=1$.
\item
Let $w_{1k}=w_{2k}=w_{3k}=0$ for $k\ge 3$.
\item
For all $r\ge 4$, let $w_{r1}=w_{r2}=0$.
\item Set $w_{rk}$ to any non-negative value such that $w_{r3}+\dots+w_{rm}=1$.
\end{itemize}
It is clear that $W$ under the above construction is row-stochastic.
Besides, $g_d^{(1)}(1,2)=g_d^{(1)}(2,3)=1/3$ and $g_d^{(1)}(1,3)=0$, 
and hence $g_d^{(1)}(1,2)+g_d^{(1)}(2,3)-g_d^{(1)}(1,3)=2/3$.

Next we consider the triangle inequality for general order $r$. We will prove that:
\begin{proposition}\label{prop4}
For any row-stochastic matrix $W$ with its column-sum diagonal matrix $Q$, 
define matrix $D^{(r)}:=(d^{(r)}_{ij})$ as $D^{(r)}:=\bm{1}\bm{1}^T-(WQ^{-1}W^{T})^r$.
Then the triangular inequality 
$d^{(r)}_{ij}+d^{(r)}_{jk}\ge d^{(r)}_{ik}$ holds for any $1\le i,j,k\le n$ 
and any positive integer $r$.
\end{proposition}

\begin{proof}
The statement for $r=1$ is proved in Proposition \ref{prop1}.
We will consider the case $r=2u$ (an even number) and the 
case $r=2u+1$ (an odd number) separately.
For the case $r=2u$, denote $\widetilde{W}=(WQ^{-1}W^{T})^u$, then $\widetilde{W}$ is symmetric.
Since both $W$ and $Q^{-1}W^{T}$ are row-stochastic, $\widetilde{W}$ is actually doubly-stochastic, and
$D^{(r)}=\bm{1}\bm{1}^T-\widetilde{W}\widetilde{W}^T$.
Since $W$ is row-stochastic, then so is $\widetilde{W}=WQ^{-1}W^{T}$,
thus the corresponding column-sum diagonal matrix $\widetilde{Q}$ for $\widetilde{W}$ 
becomes an identity matrix.
By regarding $\widetilde{W}$ as a new feature matrix $W$ in 
Proposition \ref{prop1}, the triangular inequality follows immediately.

For the case $r=2u+1$, let $\overline{W}=\widetilde{W}W$, then
\begin{equation}
D^{(r)}=\bm{1}\bm{1}^T-\widetilde{W} WQ^{-1}W^T \widetilde{W}^T=\bm{1}\bm{1}^T-\overline{W}Q^{-1}\overline{W}^T.
\end{equation}
Recalling Proposition \ref{prop1}, we only need to show that $\overline{W}$ is a 
row-stochastic matrix and $Q$ is its column-sum matrix.
Since both $\widetilde{W}$ and $W$ are row-stochastic matrices, 
so is $\overline{W}=\widetilde{W}W$.
For its column-sum, since $\widetilde{W}$ is doubly-stochastic, 
$\overline{W}^T\bf{1}=W^T\widetilde{W}^T\bf{1}=W^T\bf{1}$,
thus $\overline{W}$ and $W$ share the same column-sum matrix $Q$, which completes the proof.
\end{proof}

Theorem 1 follows directly from Proposition \ref{prop2}, Proposition \ref{prop:symmetry}, Proposition \ref{prop1}, and Proposition \ref{prop4}.

\subsection*{Analysis of $g_d^{(k)}(\cdot,\cdot)$ and $r_d^{(k)}(\cdot,\cdot)$}
For the forward graph diffusion distance $g_d^{(k)}(\cdot,\cdot)$ and its reversed version $r_d^{(k)}(\cdot,\cdot)$, symmetry is no longer guaranteed.
Besides, triangle inequality need not hold in general.
We aim to find sufficient conditions for these distances to be at least quasi-metametrics.
A counter-example to the triangle inequality is provided by these 3 objects and 2 features: $W=[1,0;2,6;0,12]$.
It is straightforward to check that $g_d^{(1)}(1,2)+g_d^{(1)}(2,3)-g_d^{(1)}(1,3)=1/3+1/2-1<0$,
and also $r_d^{(1)}(3,2)+r_d^{(1)}(2,1)-g_d^{(1)}(3,1)=1/3+1/2-1<0$.
The reason for failure of the triangle inequality is that different features 
have distinct total sums of features $p_{ii}$.
Now we are in a position to prove Theorem 2:

\begin{proof}
Based on the discussion in Section \ref{sec:basic}, we only need to prove 
the triangle inequality.
Again, we only need to prove $g_d^{(1)}(1,2)+g_d^{(1)}(2,3)\ge g_d^{(1)}(1,3)$ without 
loss of generality.
Following the proof in Proposition \ref{prop1}, it suffices to show that
\begin{equation}
\sum_{k=1}^m\frac{(w_{1k}/p_{11}+w_{3k}/p_{22})w_{2k}-w_{1k}w_{3k}/p_{11}}{q_{kk}}\le 1.
\end{equation}
Following the same argument in the proof of Proposition \ref{prop1}, we have
\begin{eqnarray*}
&&\sum_{k=1}^m\frac{(w_{1k}/p_{11}+w_{3k}/p_{22})w_{2k}-w_{1k}w_{3k}/p_{11}}{q_{kk}}\\
&\le&\sum_{k=1}^m \frac{(w_{1k}/p_{11}+w_{3k}/p_{22})w_{2k}}{w_{1k}+w_{2k}+w_{3k}} 
 \le \frac{2}{3}\frac{\max p_{ii}}{\min p_{ii}}=1,
\end{eqnarray*}
which completes the proof of Theorem 2.
\end{proof}

\end{document}